\documentclass[conference]{IEEEtran}
\IEEEoverridecommandlockouts
\usepackage[small]{caption}
\usepackage[utf8]{inputenc}
\usepackage{amssymb,amsmath,amsthm,amsfonts}
\usepackage{comment}
\usepackage{multirow}
\usepackage{mathtools}
\usepackage[linesnumbered,ruled,vlined]{algorithm2e}
\usepackage{graphicx}
\def\BibTeX{{\rm B\kern-.05em{\sc i\kern-.025em b}\kern-.08em
    T\kern-.1667em\lower.7ex\hbox{E}\kern-.125emX}}
\usepackage{footnote}    
\usepackage{color}   
\usepackage{booktabs}

\usepackage[square,numbers]{natbib}
\usepackage[utf8]{inputenc} 
\usepackage[T1]{fontenc}    
\usepackage{hyperref}       
\usepackage{url}            
\usepackage{booktabs}       
\usepackage{amsfonts}       
\usepackage{nicefrac}       
\usepackage{microtype}      
\usepackage[linesnumbered,ruled,vlined]{algorithm2e}

\usepackage{array, makecell}

\usepackage{graphicx}
\usepackage{comment}
\usepackage{multirow}
\usepackage{xcolor}
\usepackage{amsmath}
\usepackage{subcaption}
\newcommand*{\A}[1]{\textcolor{blue}{#1}} 
\usepackage{wrapfig}
\usepackage{array}
\usepackage{ragged2e}
\newcolumntype{L}[1]{>{\raggedright\let\newline\\\arraybackslash\hspace{0pt}}m{#1}}
\newcolumntype{C}[1]{>{\centering\let\newline\\\arraybackslash\hspace{0pt}}m{#1}}
\newcolumntype{R}[1]{>{\raggedleft\let\newline\\\arraybackslash\hspace{0pt}}m{#1}}

\usepackage{amsthm}

\newtheorem{theorem}{Theorem}
\newtheorem{corollary}{Corollary}
\newtheorem{lemma}{Lemma}

\newtheorem{definition}{Definition}

\begin{document}

\title{A Regularized Wasserstein Framework \\for  Graph Kernels}

\author{\IEEEauthorblockN{Asiri Wijesinghe, Qing Wang, and Stephen Gould}
    \IEEEauthorblockA{School of Computing, Australian National University, Canberra, Australia
    \\\{asiri.wijesinghe, qing.wang,  stephen.gould\}@anu.edu.au}
}
\vspace{1cm}

\maketitle

\begin{abstract}
We propose a learning framework for graph kernels, which is theoretically grounded on regularizing optimal transport. This framework provides a novel optimal transport distance metric, namely \emph{Regularized Wasserstein (RW) discrepancy}, which can preserve both features and structure of graphs via Wasserstein distances on features and their local variations, local barycenters and global connectivity. Two \emph{strongly convex} regularization terms are introduced to improve the learning ability. One is to relax an optimal alignment between graphs to be a cluster-to-cluster mapping between their locally connected vertices, thereby preserving the local clustering structure of graphs.
The other is to take into account node degree distributions in order to better preserve the global structure of graphs.
We also design an efficient algorithm to enable a fast approximation for solving the optimization problem. Theoretically, our framework is robust and can guarantee the convergence and numerical stability in optimization. We have empirically validated our method using 12 datasets against 16 state-of-the-art baselines. The experimental results show that our method consistently outperforms all state-of-the-art methods on all benchmark databases for both graphs with discrete attributes and graphs with continuous attributes.
\end{abstract}

\section{Introduction}
Graph kernels offer an appealing paradigm for measuring the similarity between graphs. They have been used in a wide range of fields such as chemoinformatics, bioinformatics, neuroscience, social networks, and computer vision \cite{kriege2020survey, vishwanathan2010graph}. Inspired by Haussler's framework for R-convolution kernels \cite{haussler1999convolution}, most of graph kernels have focused on comparing graphs
based on their substructures such as subtrees, cycles, shortest paths, and graphlets \cite{horvath2004cyclic,borgwardt2005shortest,shervashidze2009efficient,shervashidze2011weisfeiler}. However, due to the intriguing combinatorial nature of graphs, these methods have inherent limitations. For example, they do not take into account feature and structural distributions of graphs; they require substructures to be pre-defined based on domain-specific expertise which is not always available in practical applications.

In recent years, various learning-based graph kernels have been proposed \cite{yanardag2015deep,kriege2020survey}. Among them, several studies have attempted to cast the problem of measuring graph similarity as an instance of computing optimal transport distances for graphs in a kernel-based framework. Nikolentzos et al. \cite{nikolentzos2017matching} introduced a Wasserstein distance metric to compare graphs based on their node embeddings. Later, Togninalli et al. \cite{togninalli2019wasserstein} proposed a method of computing a Wasserstein distance between the node feature distributions of two graphs in the Weisfeiler-Lehman framework \cite{weisfeiler1968reduction}. 
Titouan et al. \cite{titouan2019optimal} combined Wasserstein and Gromov-Wasserstein distances in order to jointly leverage feature and structural information of graphs.
These recent advances have achieved state-of-the-art results for graph classification tasks.

Nevertheless, several technical challenges still remain for developing an effective optimal transport distance metric on graphs. Typically, optimal transport compares two probability distributions by moving one distribution to the other distribution in an optimal way that minimizes a total cost of transporting probability masses \cite{villani2003topics}. In viewing graphs as discrete distributions in a geometric metric space, optimal transport techniques can be used to explore the geometric nature of graphs. However, since optimal transport relies on cost functions to compare graphs but there is no ordering on vertices of a graph, a key challenge is, how to effectively define cost functions that can preserve intrinsic properties of graphs during the transport. Further, real-world graphs are often irregular and exhibit different geometric characteristics. This raises the challenge on how to develop a solid theoretical basis to ensure convergence and numerical stability for optimal transport learning on graphs. 


\vspace{0.15cm}
\noindent\textbf{Present work.~} To address these challenges, in this paper, we propose a powerful learning framework for graph kernels, namely \emph{Regularized Wasserstein} (RW) framework, which has two desired properties:
(1) it is theoretically robust with guaranteed convergence and numerical stability in optimization; (2) it effectively captures the rich information of graphs into transport costs so that graph kernels can account for intricate structures on graphs, including feature local variation, and local and global structures.
At its core, the RW framework is theoretically grounded on the idea of regularizing optimal transport. Below, we briefly discuss how the RW framework is designed to mitigate these challenges.

 \begin{figure*}[ht!] 
 \centering
    \includegraphics[width=1\textwidth]{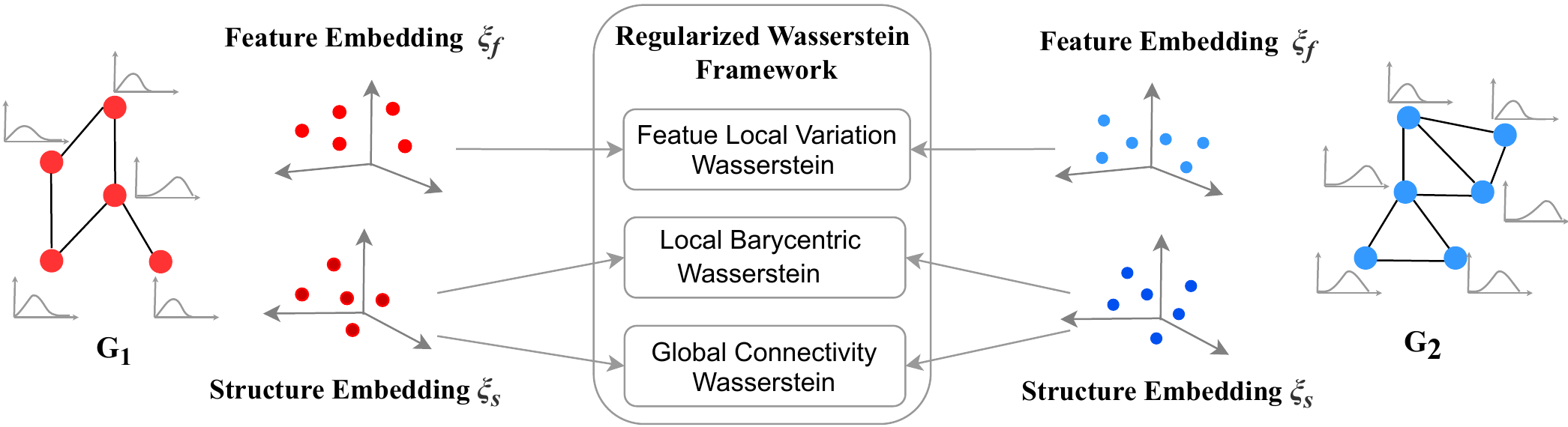}\vspace{0.1cm}
    \caption{An overview of the proposed framework for regularized Wasserstein kernels (RWKs), which unifies feature local variation, local barycentric and global connectivity Wasserstein distances based on feature and structure embeddings.}
\label{fig:overview}
\end{figure*}
 
Previous studies have considered optimal transport learning on graphs \cite{nikolentzos2017matching,togninalli2019wasserstein,maretic2019got}, which generally amounts to two kinds of graph aligning problems:
(1) aligning graphs in the same ground space and (2) aligning graphs across different ground spaces. Recent work has considered methods to jointly deal with these aligning problems based on the similarity of node features and pairwise distances, e.g., \cite{titouan2019optimal}. However, this is still inadequate due to several reasons. Firstly, these methods did not \emph{explicitly} capture the connection between features and structures into transport costs, which limits the learning ability. Secondly, these methods only considered node features based on its local feature aggregation while ignoring local clustering structures. Lastly, these methods did not exploit degree distributions when learning on pairwise distances of vertices.


In this work, we propose to capture feature local variations which quantify how features change upon the underlying structures of a graph. We explicitly incorporate feature local variations into feature similarity matrices and accordingly into a cost function to enhance optimal transport learning.
Further, we propose a new optimal transport distance metric on graphs, called \emph{Regularized Wasserstein (RW) discrepancy}. This RW discrepancy regularizes optimal transport learning to compute a distance between graphs via two \emph{strongly convex} regularization terms. One is to regularize a Wasserstein distance between graphs in the same ground space. This regularization relaxes an optimal alignment between graphs to be 
 a cluster-to-cluster mapping between their locally connected vertices, thereby preserving the local clustering structures of vertices across graphs.
The other is to regularize a Gromov-Wasserstein distance between graphs across different ground spaces using a degree-entropy KL divergence term. This regularization considers node degree distributions in order to increase the matching robustness of an optimal alignment, allowing to distribute probability masses smoothly in overlapping regions of the geometric spaces of graphs. Together with feature similarity matrices that capture features and their local variations in cost functions, our regularized optimal transport learning can preserve both local and global structures of graphs during the transport, in addition to features.

Although our framework provides a powerful optimal transport learning for graph kernels, the corresponding optimization problem is NP-hard and thus computationally difficult in the general case, due to its non-convexity and combinatorial nature. To circumvent this problem, we design an efficient algorithm, namely \emph{Sinkhorn Conditional Gradient (SCG)}, which reaps the computational benefits of the proposed strongly convex regularization terms and extends the conditional gradient with \emph{Sinkhorn-knopp} matrix scaling \cite{knight2008sinkhorn} to enable a fast approximation for solving the optimization problem. We theoretically analyse the convergence properties of SCG and prove the upper bound of its minimal suboptimality gap.

\vspace{0.2cm}
\noindent\textbf{Contributions.~} The contributions of this work are as follows.
\begin{itemize}
\item[(1)] We propose a theoretically robust class of graph kernels (i.e., RWKs) based on a new optimal transport distance metric which optimises graph aligning problems in the same or across different ground spaces by exploiting strongly convex regularisation. 
\item[(2)] We improve the geometric representation of graphs by incorporating feature local variations into similarity matrices, which can explicitly preserve the connection between features and structures of a graph. 
\item[(3)] We devise a fast and numerically stable algorithm to solve the optimisation problem and theoretically prove the suboptimal gap of our algorithm converges at the rate of $O(\frac{1}{\sqrt{k}})$ where $k$ is the number of iterations.

\end{itemize}We have evaluated our method for graph classification tasks on 12 benchmark datasets, including both graphs with discrete attributes and graphs with continuous attributes. The results  demonstrate the effectiveness of our method on real-world graphs, i.e., considerably and consistently outperforming all the state-of-the-art methods on all benchmark datasets. 
\section{Related Work}
Graph kernels have been extensively studied in the past years (see the survey by Kriege et al. \cite{kriege2020survey}). Let $\mathcal{G}$ be a non-empty set of graphs. A kernel function $\kappa: \mathcal{G} \times \mathcal{G} \rightarrow \mathbb{R}$ is defined s.t. there exists a map $\phi : \mathcal{G} \rightarrow \mathcal{H}$ with $\kappa(G_i,G_j)=\big<\phi(G_i), \phi(G_j)\big>_{\mathcal{H}}$, where $\mathcal{H}$ refers to a reproducing kernel Hilbert space (RKHS) \cite{haussler1999convolution}. Traditionally, $\kappa$ must be symmetric and positive semidefinite (i.e. a PSD kernel) because this enables kernel-based learning methods such as SVM to solve classification problems efficiently by convex quadratic programming \cite{gu2012learning}. However, many practical applications may produce indefinite kernels \cite{qamra2005enhanced,roth2003optimal,noma2002dynamic} and cannot be theoretically supported in the traditional kernel setting. For example, standard SVM learning with an indefinite kernel is a nonconvex optimization problem \cite{gu2012learning}. 
Therefore, several approaches have been proposed to address the issues of indefinite kernels, e.g., applying spectral transformations to indefinite kernels, reformulating a kernel learning problem into a convex optimization problem, etc. \cite{pekalska2001generalized,roth2003optimal,oglic2018learning,loosli2015learning,chen2009learning}. In this work, our proposed RWK kernels are indefinite. Inspired by the previous work \cite{luss2008support}, we treat indefinite kernels as noisy observations of a true PSD kernel (see a detailed discussion in Section \ref{subsec:graph_kernels}).

Optimal transport has recently received revived
interest from the machine learning community, due to its elegant way to measure the distance between two probability spaces. Following \cite{memoli2011gromov}, \citet{peyre2016gromov} introduced a Gromov-Wasserstein distance to compare pairwise similarity matrices from different metric spaces. Later, several studies have devoted to distance metrics for graphs. \citet{titouan2019optimal} proposed a fused Gromov-Wasserstein distance to combine Wasserstein and Gromov-Wasserstein distances in order to jointly leverage feature and structural information of graphs. To capture global graph structure, \citet{maretic2019got} proposed a Wasserstein distance between graph signal distributions by resorting to graph Laplacian matrices. This method was initially constrained to graphs of the same sizes, but recently extended to graphs of different sizes by formulating graph matching as a one-to-many assignment problem \cite{maretic2020wasserstein}.  \citet{xu2019gromov} proposed to jointly align graphs and learn node embeddings using a Gromov-Wasserstein distance. 
To reduce computational complexity, Gromov-Wasserstein distances are often computed using a Sinkhorn algorithm \cite{cuturi2013sinkhorn,sinkhorn1967diagonal}. Recently, a scalable method was proposed in \cite{xu2019scalable} to recursively partition and align large-scale graphs based on a Gromov-Wasserstein distance. 

\vspace{0.2cm}
\section{Regularized Optimal Transport}

Let $G = (V, E)$ be an undirected graph where $V$ is a set of vertices and $E$ is a set of edges. 
A \emph{feature embedding} function $\xi_f: V\rightarrow \mathbb{R}^{m}$ associates each vertex with a feature representation in a metric space $(\mathbb{R}^{m}, d_f)$. A \emph{structure embedding} function $\xi_s: V\rightarrow \mathbb{R}^{k}$ associates each vertex with a structural representation in a metric space $(\mathbb{R}^{k}, d_s)$.

Now, we define the notion of discrete probability distribution for graphs \cite{titouan2019optimal}. Let $\Sigma_n := \{\mu \in \mathbb{R}^{n}_{+} : \sum_{i}^{n} \mu_i=1\}$ be a histogram which encodes the weight $\mu_i$ of each vertex $v_i\in V$ according to some prior information, e.g. uncertainty or relative importance. We set $\mu =(1/n) \textbf{1}_n$ (i.e., uniform distribution) if no prior information is available, where $\textbf{1}_n$ is a $n$-dimensional vector of ones. Then, a graph $G$ can be represented as a discrete probability distribution in the product space of $(\mathbb{R}^{m}, d_f)$ and $(\mathbb{R}^{k}, d_s)$, where $\delta$ refers to a Dirac function that corresponds to the feature and structure embeddings of vertices:
\begin{equation}\label{equ:graph-ot}
p= \sum_{i=1}^{n}\mu_i\delta(\xi_f(v_i),\xi_s(v_i)).
\end{equation}

Given two graphs $G_1$ and $G_2$ with $n_1$ and $n_2$ vertices, respectively, we denote their discrete probability distributions as $\mu \in \Sigma_{n_1}$ and $\nu \in \Sigma_{n_2}$. The set of probabilistic couplings between $G_1$ and $G_2$ is defined as:
\begin{equation*}
\label{equ:stochastic-mat}
\pi(\mu,\nu) =\Big\{\gamma \in \mathbb{R}_+^{{n_1} \times {n_2}}\hspace{0.1cm}|\hspace{0.1cm} \gamma \textbf{1}_{n_2}=\mu, \gamma^T \textbf{1}_{n_1}=\nu \Big\}.
\end{equation*}

In this work, we aim to formalize a regularized optimal transport problem for graph kernel learning by finding an optimal coupling $\hat{\gamma}$ between two graphs:
\begin{equation}\label{equ:entropy-reg}
\hat{\gamma} = \underset{\gamma \in \pi(\mu,\nu)}{\text{argmin}} \big<\gamma, \mathbf{C}\big>_F + \lambda\Theta(\gamma),
\end{equation}
where $\mathbf{C}$ is a cost function matrix which measures the cost of moving a probability mass from $\mu$ to $\nu$, $\big<.,.\big>_F$ denotes the Frobenius dot product, $\lambda\in [0, 1]$ and $\Theta(\gamma)$ is a regularizer on $\gamma$. Then, given a set of graphs $\mathcal{G}$, we define a graph kernel: $\mathcal{G}\times \mathcal{G} \rightarrow \mathbb{R}$ where the kernel value for each pair of graphs in $\mathcal{G}$ is defined upon their optimal transport distance. 

We will first introduce graph similarity matrices used for cost functions in Section \ref{sec:optimal-transport}, and then present in detail how to define such a regularized optimal transport problem for a graph kernel in Section \ref{sec:proposed-method}. 

\vspace{0.2cm}
\section{Graph Similarity Matrices}
\label{sec:optimal-transport}
In this section we discuss the feature and structural representations of graphs and several cost functions for optimal transport learning on graphs. 

\subsection{Feature Similarity}
Following the previous work \cite{maretic2019got}, we consider features residing on vertices as graph signals. For a graph $G=(V,E)$, a \emph{graph signal} is a mapping $V\rightarrow \mathbb{R}$ that associates a feature to a vertex. Thus, each graph has a \emph{graph signal matrix} $\mathbf{X} \in \mathbb{R}^{n\times m}$, where $n=|V|$ is the number of vertices in the graph and each vertex $v_i$ is associated with graph signals $x_i\in \mathbb{R}^m$.


To quantify how graph signals change from a vertex to its neighboring vertices, we formulate the notion of  feature local variation.
Let $\mathbf{L}=\mathbf{I}-\mathbf{D}^{-1/2} \mathbf{AD}^{-1/2}$ be the normalised graph Laplacian of $G$, where $\mathbf{D}$ is the diagonal matrix, $\mathbf{A}$ is the adjacency matrix and $\mathbf{I}$ is the identity matrix. Then the \emph{local variation matrix} of $G$ is defined as: 
\begin{equation}\label{equ:tv}
\Delta(\textbf{X})= \left\vert \textbf{X} - \frac{\mathbf{L}^j\mathbf{X}}{\lambda_{max}(\mathbf{L})} \right\vert.
\end{equation}
$\mathbf{L}^j\textbf{X}$ refers to aggregated graph signals of all vertices in $G$ within the j-hop neighborhood. $\lambda_{max}(\mathbf{L})$ is the maximum eigenvalue of $\mathbf{L}$, which normalises $\mathbf{L}^j\mathbf{X}$ to ensure the numerical stability. $\Delta(\textbf{X})$ represents the local variations of features  computed by taking the difference between the original graph signal matrix $\textbf{X}$ and the aggregated graph signal matrix $\mathbf{L}^j\textbf{X}$.  



Let $x_i\in \mathbb{R}^m $ and $\Delta(x_i)\in \mathbb{R}^m $ refer to the graph signals of a vertex $v_i$ and its local variation in $G$, respectively. Then, each vertex $v_i$ corresponds to a feature embedding vector $a_i=\xi_f(v_i)\in \mathbb{R}^{2m} $ such that $a_i=x_i \oplus \Delta(x_i)$, where $\oplus$ refers to the concatenation. Given two graphs $G_1$ and $G_2$, a \emph{feature similarity matrix} between $G_1$ and $G_2$ is defined upon the concatenation of their graph signals and local variations, i.e., $\mathbf{C}^V(i,j)=(d_f(a_i, a_j))_{i,j}\in \mathbb{R}^{n_1 \times n_2}$, where $a_i$ and $a_j$ are the feature embedding vectors of the $i\text{-}th$ vertex of $G_1$ and $j\text{-}th$ vertex of $G_2$, respectively.

\begin{figure*}[ht!] 
 \centering
    \includegraphics[width=1\textwidth]{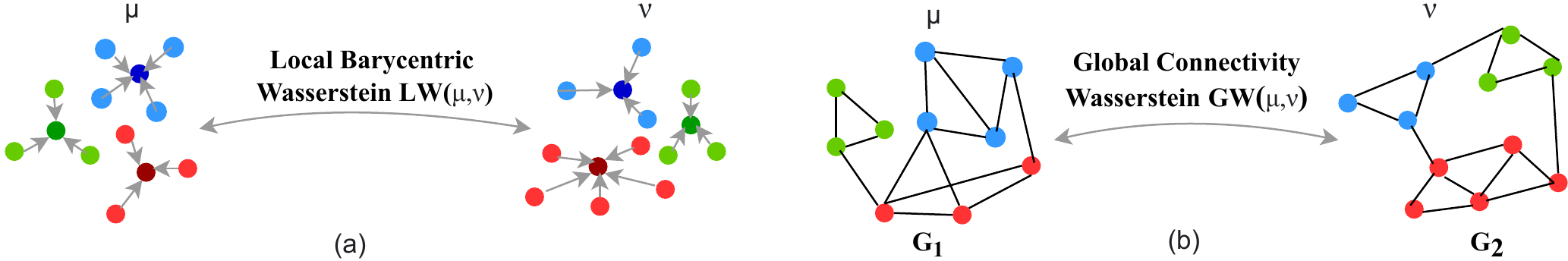}
    \caption{(a) shows the local barycentric Wasserstein distance that transports each vertex in $\mu$ to a spatially localized barycenter of its corresponding neighbors in $\nu$ and vice versa; (b) shows the global connectivity Wasserstein distance that captures the pairwise similarity between vertices under the preservation of degree distributions.}
\label{fig:overview}
\end{figure*}

\subsection{Structure Similarity}
\label{subsec:node-embeddings}
For each vertex $v_i\in V$ in a graph, we associate it with a node embedding vector $e_i=\xi_s(v_i)\in \mathbb{R}^k$. The node embeddings are learned using heat kernel random walks with graph attention. More specifically, we construct a probability transition matrix $\mathbf{M}=e^{-t\mathbf{L}}$, where $t$ is the length of random walks and $\mathbf{L}$ is the graph Laplacian. A graph attention mechanism guides the sampling process of random walks to optimize an objective {Negative Log Graph Likelihood} \cite{abu2018watch}.

Based on the node embeddings, we consider the following two kinds of structure similarity:
\begin{itemize}
    \item[(1)] \emph{Neighbourhood similarity.~} For two graphs $G_1$ and $G_2$, we define a \emph{neighbourhood similarity matrix} as $\mathbf{C}^N(i,j)=(d_s(e_{i}, e_{j}))_{i,j}\in \mathbb{R}^{n_1 \times n_2}$ where $e_{i}$ and $e_{j}$ represent the node embeddings of the $i\text{-}th$ vertex of $G_1$ and the $j\text{-}th$ vertex of $G_2$, respectively.
    \item[(2)] \emph{Pairwise similarity.~} For a graph $G$, we construct a \emph{pairwise similarity matrix} by $\mathbf{C}^P(i,j)=(d_s(e_{i}, e_{j}))_{i,j}  \in \mathbb{R}^{n \times n}$, where $e_{i}$ and $e_{j}$ represent the node embeddings of the $i\text{-}th$ vertex and and the $j\text{-}th$ vertex of $G$. Let $\mathbf{C}^P_1 \in \mathbb{R}^{n_1 \times n_1}$ and $\mathbf{C}^P_2 \in \mathbb{R}^{n_2 \times n_2}$ represent the pairwise similarity matrices of two graphs $G_1$ and $G_2$, respectively. Then, the pairwise similarity between $G_1$ and $G_2$ is defined as a $4 \text{-} \emph{dimensional}$ tensor:
\begin{equation*}\label{equ:pairwis}
L_2(\mathbf{C}^P_1(i,j),\mathbf{C}^P_2(k,l))= \frac{1}{2} |\mathbf{C}^P_1(i,j) - \mathbf{C}^P_2(k,l)|^2.
\end{equation*}
\end{itemize} 


\vspace{0.1cm}
\section{Regularized Wasserstein Framework}
\label{sec:proposed-method}
In this section, we introduce a novel optimal transport framework for graphs. This framework can preserve local and global graph structures by jointly optimising two regularized optimal transports on graphs: (1) local barycentric Wasserstein distance; (2) global connectivity Wasserstein distance. We discuss these two kinds of Wasserstein distances in turn. 

\subsection{Local Barycentric Wasserstein Distance}\label{subsec:wasserstein} 
We first propose a local-structure-preserving optimal transport based on Laplacian regularization \cite{ferradans2013regularized,flamary2014optimal}. To preserve the local structure of graphs, we observe that a relaxed mapping (i.e., cluster-to-cluster) between locally connected vertices of two graphs is often more desirable than a strict one-to-one correspondence between vertices of two graphs. Thus, we design a regularization term $\Theta_w(\gamma)$ under a relaxation of transport mass conservation \cite{fournier2015rate} to regularize a Wasserstein distance defined on the neighbourhood similarity matrix $\mathbf{C}^{N}$: 
\begin{equation}\label{equ:graph-wasserstein}
LW(\mu,\nu) = \underset{\gamma \in \pi(\mu,\nu)}{\text{min}} \big<\gamma, \mathbf{C}^{N}\big>_F +  \Theta_w(\gamma),
\end{equation}
where $\big<.,.\big>_F$ denotes the Frobenius dot product.

In the following, we discuss how $\Theta_w(\gamma)$ is designed. Essentially, $\gamma(i, j)$ indicates how much the probability mass of the $i \text{-} th$ vertex in one graph $\mu$ is transported to the $j \text{-} th$ vertex in the other graph $\nu$. Thus, we define a transport map $T$ from $\mu$ to $\nu$ by mapping the node embedding of each vertex $e_i^{\mu}$ in $\mu$ to a weighted average $\hat{e}_i^{\mu}$ of the node embeddings of vertices in $\nu$:  
\begin{equation}\label{equ:barycenter-mapping}
\hat{e}_i^{\mu}=T(e_i^{\mu}) = \frac{\sum_{j=1}^{n_2} \gamma(i,j) e_j^{\nu}}{\sum_{j=1}^{n_2} \gamma(i,j)}.
\end{equation}

Let $\mathbf{E}_{\mu} \in \mathbb{R}^{n_{1} \times k}$ (resp. $\mathbf{E}_{\nu} \in \mathbb{R}^{n_{2} \times k}$) be a node embedding matrix of $\mu$ (resp. $\nu$). We thus have the following matrix of local barycentric embeddings:
\begin{equation}\label{equ:vector-barycenter-mapping-xs}
 \hat{\mathbf{E}}_{\mu} = T(\mathbf{E}_{\mu}) = (diag(\gamma \textbf{1}_{n_{2}}))^{-1} \gamma \mathbf{E}_{\nu},
\end{equation}
where $diag(.)$ is a diagonal matrix in $\mathbb{R}^{n_{1} \times n_{1}}$. To preserve the local structure of vertices in $\mu$ under $T$, we define a spatially localized barycentric term as the \emph{source regularization}: 
\begin{equation}\label{equ:graph-reg1}
\begin{split}
\Omega_{\mu}(\gamma) &= \frac{1}{n_{1}^2} \sum_{i,j} a_{i,j} |\!|\hat{e}^{\mu}_i - \hat{e}^{\mu}_j|\!|_2^2\\
&= \frac{1}{n_{1}^2} tr(\hat{\mathbf{E}}_{\mu}^T \mathbf{L}_{\mu} \hat{\mathbf{E}}_{\mu}).\\
\end{split}
\end{equation}
$\mathbf{L}_{\mu}$ is the graph Laplacian and $\mathbf{A}_{\mu}=(a_{i.j})_{i,j=1}^{n_1}$ is the adjacency matrix of $\mu$. 
When $\mu$ and $\nu$ are uniform distributions,
$\hat{\mathbf{E}}_{\mu} =  n_{1} \gamma \mathbf{E}_{\nu}$ 
and thus 
\begin{equation}\label{equ:graph-reg2}
\Omega_{\mu}(\gamma) = tr(\mathbf{E}_{\nu}^T \gamma^T \mathbf{L}_{\mu} \gamma \mathbf{E}_{\nu}).
\end{equation}
Similarly, we define a spatially localized barycentric term $\Omega_{\nu}(\gamma)$ as the \emph{target regularization} to preserve the local structure of vertices in $\nu$ under the transport map $T^{-1}$. By $\Omega_{\mu}(\gamma)$ and $\Omega_{\nu}(\gamma)$, we obtain the following regularization term to constrain local barycentric Wasserstein distance, where $0 \le \lambda_{\mu}, \lambda_{\nu} \le 1$: 
\begin{equation}\label{equ:graph-reg3}
\Theta_w(\gamma) = \lambda_{\mu} \Omega_{\mu}(\gamma) + \lambda_{\nu} \Omega_{\nu}(\gamma) + \frac{\rho}{2}||\gamma||_F^2.
\end{equation}



This regularization term enables us to avoid the strict mass conservation (i.e, a bijective mapping between $\mu$ and $\nu$) because each vertex in $\mu$ is transported to a spatially localized barycenter of its corresponding neighbors in $\nu$ and vice versa. A penalty term $||\gamma||_F^2$ is introduced to smooth the transport mass conservation. The parameter $\rho \in (0,1]$ controls the degree of smoothness.
  

\begin{lemma}\label{lem:local}
$LW(\mu,\nu)$ is strongly convex and smooth w.r.t. $\gamma$.  
\end{lemma}

\begin{proof}
Let $f_1(\gamma) = \lambda_{\mu} \Omega_{\mu}(\gamma) + \lambda_{\nu} \Omega_{\nu}(\gamma)$ and $f_2(\gamma) = \frac{\rho}{2}||\gamma||_F^2$.
The Hessian of $\Omega_{\mu}(\gamma)$ is,
\begin{equation}\label{equ:hessian-omega-mu}
\nabla^2 \Omega_{\mu}(\gamma) = \mathbf{L_{\mu}} \otimes \mathbf{E_{\nu} E_{\nu}^T} + \mathbf{L_{\mu}^T} \otimes \mathbf{E_{\nu} E_{\nu}^T},
\end{equation}
where $\otimes$ denotes the Kronecker product. $\mathbf{L_{\mu}}$ is positive semi-definite since its eigenvalues are non-negative. We also have $z^T(\mathbf{E_{\nu}E_{\nu}^T})z=||\mathbf{E_{\nu}^T} z||_2^2 \ge 0$ for every $z \neq 0$ and $z \in \mathbb{R}^{n_2 \times 1}$, which is positive semi-definite. Thus, $\mathbf{L_{\mu}} \otimes \mathbf{E_{\nu}E_{\nu}^T}$ is positive semi-definite since the Kronecker product of two positive semi-definite matrices is positive semi-definite \cite{schacke2004kronecker}. Therefore, $\Omega_{\mu}(\gamma)$ is convex, and similarly, we can show $\Omega_{\nu}(\gamma)$ is convex. Hence, $f_1(\gamma)$ is convex w.r.t $\gamma$. Since the function $||\gamma||_F^2$ is quadratic w.r.t $\gamma$, the Hessian of $f_2(\gamma)$ is positive definite. Hence, $f_2(\gamma)$ is strongly convex. Then, the sum of $f_1(\gamma)+f_2(\gamma)$ (i.e., $\Theta_w(\gamma)$) is $\rho$-strongly convex. 
Since $f_1(\gamma)$ is positive semi-definite and $f_2(\gamma)$ is positive definite, $\Theta_w(\gamma)$ is positive definite. Hence, $\Theta_w(\gamma)$ is $L$-smooth for some constant $L>0$.

Since $\big<\gamma, \mathbf{C}^{N}\big>_F$ is convex and $\Theta_w(\gamma)$ is strongly convex and smooth, $LW(\mu,\nu)$ is strongly convex and smooth.
\end{proof}


\subsection{Global Connectivity Wasserstein Distance}\label{subsec:gw} 
To preserve
the global structure of graphs during the transport, such as structure connectivity, a straightforward approach is to use a Gromov-Wasserstein discrepancy based on pairwise similarity between vertices. However, solving such an unregularized Gromov-Wasserstein optimization problem on $\big<\gamma, L_{2}(\mathbf{C}^P_{\mu}, \mathbf{C}^P_{\nu})\otimes \gamma\big>_F$ may lead to a sparse coupling matrix $\gamma$, i.e. the entries of $\gamma$ become mostly zero. As a result, only few vertices between two graphs can be matched. Further, the degree distributions between graphs need to be considered for preserving structure connectivity, Thus, we design a \emph{degree-entropy regularization term} $\Theta_g(\gamma)$ to regularize a Gromov-Wasserstein distance on the pairwise similarity matrix $\mathbf{C}^{P}$: 
\begin{equation*}\label{equ:graph-gw}
GW(\mu,\nu) = \underset{\gamma \in \pi(\mu,\nu)}{\text{min}} \big<\gamma, L_{2}(\mathbf{C}^P_{\mu}, \mathbf{C}^P_{\nu})\otimes \gamma\big>_F - \lambda_g \Theta_g(\gamma), 
\end{equation*}
where $\lambda_g \in (0,1]$ and $\big<\gamma, L_{2}(\mathbf{C}^P_{\mu}, \mathbf{C}^P_{\nu})\otimes \gamma\big>_F=\sum_{i,j,k,l} L_{2}(\mathbf{C}^P_{\mu}(i,j), \mathbf{C}^P_{\nu}(k,l)) \gamma(i,k) \gamma(j,l)$. Specifically, we define $\Theta_g(\gamma)$ as a KL divergence between $\gamma$ and a prior node degree distribution $\gamma'$:
\begin{equation}\label{equ:kl-divergence}
\Theta_g(\gamma)=KL(\gamma\|\gamma')= \sum_{i,j} \gamma(i,j) log \Big(\frac{\gamma(i,j)}{\gamma'(i,j)} \Big). 
\end{equation}
Let $D_{\mu} \in \mathbb{R}^{n_1}$ and $D_{\nu} \in \mathbb{R}^{n_2}$ represent the node degree vectors of graphs $G_1$ and $G_2$, respectively.
\begin{equation}
\begin{split}
\gamma'(i,j)=\frac{\Tilde{\gamma}(i,j)}{||\sum_{j} \Tilde{\gamma}(i,j)||_1}\\
\Tilde{\gamma}(i,j) = 1-\frac{|D_{\mu}^i - D_{\nu}^j|}{\text{max}\{D_{\mu}^i, D_{\nu}^j\}}
\end{split}
\end{equation}

Note that, depending on how pairwise similarity matrices are defined, different kinds of global structures can be preserved. When $\mathbf{C}^P_{\mu}$ and $\mathbf{C}^P_{\nu}$ are shortest path distance matrices, we preserve the connectivity structure of graphs. Other options include adjacency matrices and graph Laplacians~\cite{schieber2017quantification}. 


 
\begin{lemma}\label{lem:kl}
$KL(\gamma||\gamma')$ is strongly convex w.r.t $\gamma$.
\end{lemma}
\begin{proof} 
We can compute the Hessian of $KL(\gamma||\gamma')$ as follows,
\begin{equation}\label{equ:kl_hessian}
\nabla^2 KL(\gamma||\gamma') = diag \Big(\frac{1}{\gamma(i,j)}\Big),
\end{equation}
where $\gamma(i,j) \in [0,1]$. Since a function $f$ is $\sigma$-\emph{strongly convex} iff there exists a constant $\sigma > 0$ s.t. its Hessian satisfies $\nabla^2 f(\gamma) \succeq \sigma \mathbf{I}$, $\forall \gamma \in$ dom f, where $\mathbf{I}$ refers to an identity matrix,
$KL(\gamma||\gamma')$ is $1$-strongly convex because $z^T (\nabla^2 KL(\gamma||\gamma')) z \ge  \sigma||z||^2$ and $\sigma=1$. 
\end{proof}
 Although $GW(\mu,\nu)$ remains non-convex, the strong convexity of $KL(\gamma||\gamma')$ enables better optimization convergence (will be discussed further in Section \ref{subsec:gcg}).

\subsection{RW Discrepancy}\label{subsec:gcg}
In the following, we present the \emph{Regularized Wasserstein} (RW) discrepancy to preserve both features and structure of graphs. The main idea is to consider local barycentric and global connectivity Wasserstein distances, as well as Wasserstein distance for features and their local variations, in a unified framework.  We also discuss our optimization technique and analyze the theoretical properties.

Let $\beta_1,\beta_2 \in (0,1]$ and $\mathbf{C}^{V}$ be a feature similarity matrix containing the information of features and their local variations. Formally, the RW discrepancy is defined as follows: 
\begin{equation}\label{equ:graph-gcg-wasserstein}
\begin{split}
RW(\mu,\nu) =  &\underset{\gamma \in \pi(\mu,\nu)}{\text{min}} \hspace{0.15cm}\big<\gamma, \mathbf{C}^{V}\big>_F  \\ &+\beta_1 LW(\mu,\nu) + \beta_2 GW(\mu,\nu). 
\end{split}
\end{equation}
In a nutshell, the RW discrepancy derives an optimal coupling $\gamma$ by minimizing a linear combination of costs of transporting graph features and their local variations, transporting vertices and transporting edges across two graphs.



Solving an unregularized Gromov-Wasserstein optimization problem in its full generality is known to be NP-hard \cite{alvarez2018gromov, vayer2019sliced}. The optimization problem for Eq. \ref{equ:graph-gcg-wasserstein} is thus also NP-hard. Therefore, the convergence to the optimality of RW is a non-trivial and difficult problem. Below, we present a solution to tackle this difficult problem.



Firstly, we transform the optimization problem for Eq.~\ref{equ:graph-gcg-wasserstein} into an equivalent problem with the following form of objective:
\begin{equation}\label{equ:gcg-objective}
\underset{\gamma \in \pi(\mu,\nu)}{\text{min}} H(\gamma) =  \underset{\gamma \in \pi(\mu,\nu)}{\text{min}} f(\gamma) + g(\gamma) {-h(\gamma)},
\end{equation}
where we have:
\begin{equation*}\label{equ:gcg-sinkhorn-1}
\begin{split}
f(\gamma) =& \big<\gamma, \mathbf{C}^V\big>_F + \beta_1 LW(\mu,\nu);  \\
g(\gamma) =& \big<\gamma,\beta_2(L_{2}(\mathbf{C}^P_{\mu}, \mathbf{C}^P_{\nu})\otimes \gamma)\big>_F; \\ 
h(\gamma) =& \beta_2(\lambda_g \Theta_g(\gamma)).
\end{split}
\end{equation*}

\begin{algorithm}[ht]
\SetAlgoLined
 initialize i=0, $\gamma^0 \gets \mu \nu^T$, and $c^0 \gets H(\gamma^0)$\\
 \While{$i \le t$}{
  $i \gets i+1$\\
  $\nabla H(\gamma) \gets$ Gradient of $H(\gamma)$ w.r.t $\gamma^{(i-1)}$\\
  $\hat{\gamma}^{(i-1)} \gets$ $\emph{Sinkhorn-knopp}$ ($\mu$, $\nu$, $\nabla H(\gamma)$, $\lambda$, $b$)\\
  $\Delta \gamma \gets \hat{\gamma}^{(i-1)} \text{-} \gamma^{(i-1)}$\\
  $\alpha^{(i)}, c^{(i)} \gets $ \emph{Line-search} ($\gamma^{(i-1)}$, $\Delta \gamma$, $\nabla H(\gamma)$, $c^{(i-1)}$) w.r.t. Eq.~\ref{equ:gcg-objective}\\ 
  $\gamma^{(i)} \gets \gamma^{(i-1)} + \alpha^{(i)} \Delta \gamma$\\
  $\delta^{(i-1)} \gets \Big<\Delta \gamma, -\nabla H(\gamma) \Big>_F$\\
  \If{$\delta^{(i-1)} \le \epsilon $}{stop}
 }
 \caption{Training for RW Discrepancy}\label{alg:SCG}
\end{algorithm}

Then, we design a training algorithm for RW discrepancy, namely \emph{Sinkhorn Conditional Gradient} (SCG), based on Conditional Gradient \cite{jaggi2013revisiting}, which is described in Algorithm \ref{alg:SCG}. 
The main idea is to linearize the composite objective function in Eq.~\ref{equ:gcg-objective}, where $t$ is the maximum number of iterations for SCG and $b$ is the maximal number of Sinkhorn iterations. In each iteration, we compute an optimal coupling matrix $\hat{\gamma}^{(i-1)}$ based on the gradient of $H(\gamma)$ by \emph{Sinkhorn-knopp}, where $\lambda\in[0,\infty]$ and obtain the descent direction $\Delta \gamma$. Then, we use \emph{Line-search} to determine the step size $\alpha^{(i)}$ based on the gradient of $H(\gamma)$ along the descent direction $\Delta \gamma$. The algorithm terminates if the suboptimality gap converges under a threshold $\epsilon$, i.e., $\delta^{(i-1)} \le \epsilon$. 

The gradients of $f(\gamma)$ and $g(\gamma)$ are calculated as follows: 
\vspace{0.05cm}
\begin{equation*}\label{equ:gcg-sinkhorn-derivative}
\begin{split}
\nabla f(\gamma) =& \mathbf{C}^V+\beta_1(\mathbf{C}^{N}) +\\ & \beta_1(\lambda_{\mu} \nabla\Omega_{\mu}(\gamma) + \lambda_{\nu} \nabla\Omega_{\nu}(\gamma)+\rho \gamma);\\
\nabla g(\gamma) =& 2 \beta_2(L_{2}(\mathbf{C}_{\mu}^P, \mathbf{C}_{\nu}^P)\otimes \gamma);\\
\nabla h(\gamma) =& \beta_2(\lambda_g(1+log(\gamma)-log(\gamma'))),
\end{split}
\end{equation*}\vspace{0.1cm}
where
\begin{equation*}\label{equ:partial-derivative-omega-s}
\begin{split}
\nabla\Omega_{\mu}(\gamma) =& \frac{\partial{(\Omega_{\mu}(\gamma))}}{\partial{\gamma}}=\mathbf{L}_{\mu}^T \gamma \mathbf{E}_{\nu} \mathbf{E}_{\nu}^T + \mathbf{L}_{\mu} \gamma \mathbf{E}_{\nu} \mathbf{E}_{\nu}^T;
\\\nabla\Omega_{\nu}(\gamma) =& \frac{\partial{(\Omega_{\nu}(\gamma))}}{\partial{\gamma}}=\mathbf{E}_{\mu} \mathbf{E}_{\mu}^T \gamma \mathbf{L}_{\nu}^T + \mathbf{E}_{\mu} \mathbf{E}_{\mu}^T \gamma \mathbf{L}_{\nu}.
\end{split}
\end{equation*}

SCG has nice convergence properties. It is guaranteed to converge to a stationary point. Below, we define suboptimality gap \cite{jaggi2013revisiting} for SCG and present the theoretical results. 

\begin{definition}[Suboptimality gap]\label{d_supoptimality_gap} For each $i$-th iteration of SCG, the suboptimality gap $\delta_i$ is defined by
\begin{equation}\label{equ:supoptimality_gap}
\delta_i = \underset{\hat{\gamma} \in \pi(\mu,\nu)}{\text{max}} \Big<(\gamma - \hat{\gamma}), \nabla H(\gamma) \Big>_F.
\end{equation}
\end{definition}

We know that, by Lemma \ref{lem:local} $f(\gamma)$ is $L$-smooth, and by the results of \cite{chapel2020partial} $g(\gamma)$ is also $L$-smooth. Thus, $f(\gamma)+g(\gamma)$ is $L$-smooth. Further, by Lemma \ref{lem:kl}, $h(\gamma)$ is strongly convex. Thus, we obtain a generalized curvature constant $C_{f+g-h} \le (L-\sigma) \cdot diam_{||.||}(\pi(\mu,\nu))^2$ where $0<\sigma<L$. By the results of Frank-Wolfe algorithm for non-convex functions \cite{lacoste2016convergence}, we have $\underset{0 \le i \le k}{\text{min}} \delta_i \le \frac{max\{2h_0, C_{f+g-h}\}}{\sqrt{k+1}}$, for $k \ge 0$. Hence, we obtain the following theorem. 


\begin{theorem}[\textbf{Convergence}]\label{t_convergence_SCG} SCG has the minimal suboptimality gap $\delta_i$ that satisfies the following condition:
\begin{equation}\label{equ:GCG_convergence_theory}
\underset{0 \le i \le k}{\text{min}} \delta_i \le \frac{max\{2h_0, (L-\sigma) \cdot diam_{||.||} (\pi (\mu,\nu))^2\}}{\sqrt{k+1}}
\end{equation}
where $\sigma=1$, $h_0=H(\gamma^0) - \underset{\gamma \in \pi(\mu,\nu)}{\text{min}} H(\gamma)$ is the initial suboptimality gap, $L$ is a Lipschitz constant of $\nabla(f+g)(\gamma)$, and $diam_{||.||} (\pi (\mu,\nu))^2$ denotes the $||.||_F$-diameter of the $\pi(\mu,\nu)$. 
\end{theorem}

Following Theorem \ref{t_convergence_SCG}, we have the corollary below.

\begin{corollary}
For SCG, the minimal suboptimality gap is $O(\frac{1}{\sqrt{k}})$ after the number $k$ of iterations. It takes at most $O(\frac{1}{\epsilon^2})$ iterations to find an approximate stationary point with a suboptimality gap smaller than $\epsilon$.
\end{corollary}

\subsection{Regularized Wasserstein Kernels (RWK)}\label{subsec:graph_kernels}

We introduce a new graph kernel, namely \emph{Regularized Wasserstein Kernel} (RWK), based on our RW discrepancy presented in Section \ref{subsec:gcg}. Given a set of graphs $\mathcal{G}$, RWK has a \emph{kernel matrix}  $\mathbf{K}\in \mathbb{R}^{|\mathcal{G}| \times |\mathcal{G}|}$ defined as
\begin{equation*}\label{equ:RJW_kernel}
\mathbf{K}_{\mu \nu} = e^{- \eta RW(\mu,\nu)},
\end{equation*}
where $\eta>0$ is a parameter, $\mu$ and $\nu$ correspond to any two graphs in $\mathcal{G}$, and $RW(\mu,\nu)$ is the RW discrepancy between $\mu$ and $\nu$ as defined in Eq.~\ref{equ:graph-gcg-wasserstein}.

Here, $\mathbf{K}$ is an indefinite kernel matrix. Following SVM with indefinite kernels introduced by Luss and d'Aspremont \cite{luss2008support}, we treat $\mathbf{K}$ as the noisy observation of a true positive semi-definite kernel (i.e., a proxy kernel). Thus, our graph classification problem with an indefinite RWK can be expressed as a robust classification problem under a perturbation of the true positive semidefinite kernel. This formulation allows us to learn support vector weights and a proxy kernel simultaneously, while penalizing the distance between the indefinite RWK and the proxy kernel in the same way as studied in \cite{luss2008support}. 



\subsection{Computational Complexity} A naive implementation of $\nabla f(\gamma)$ has the time complexity $O(N^4)$ due to the tensor-matrix multiplication in $GW(\mu,\nu)$, where $N=max\{n_1,n_2\}$. Nevertheless, as discussed in \cite{peyre2016gromov}, for a general class of loss functions, the tensor-matrix multiplication can be decomposed into matrix-matrix multiplications and the time complexity of $\nabla f(\gamma)$ can thus be reduced to $O(N^3)$. $\nabla g(\gamma)$ has the time complexity $O(N^3)$. The time complexity of the line search algorithm depends on the computation of $H(\gamma)$ in Eq. \ref{equ:gcg-objective}. Since it has the time complexity $O(N^3 + N^2 k^2)$, the total time complexity of our algorithm is $O(t(N^3 + N^2 k^2))$, where $t$ refers to the total number of iterations and $k$ is the dimension of the node embedding.
The memory complexity of our algorithm is $O(N^2)$.

Table \ref{Tab:complexity} summarizes the time and memeory complexity of several optimal transport based graph kernels. Note that $t$ is much smaller than $N$ in practice.

\begin{table}[h!]
\centering\vspace{-0cm}
\scalebox{1}{\begin{tabular}{| c| c| c|}  \hline
  \multirow{2}{*}{} Optimal Transport Based & Time & Memory \\ Graph Kernel& Complexity &Complexity\\
  \hline
      {WL-PM} \cite{nikolentzos2017matching}& $O(N^3 log(N))$ & $O(N^2)$\\\hline
    WWL \cite{togninalli2019wasserstein} & $O(N^3 log(N))$ & $O(N^2)$ \\\hline
    FGW \cite{titouan2019optimal}&  $O (t(N^3))$ & $O(N^2)$ \\\hline
    {RWK} (ours) &  $O(t(N^3 + N^2 k^2))$ &  $O(N^2)$\\
  \hline
 \end{tabular}}
 \caption{A summary of time and memory complexities. \label{Tab:complexity}}\vspace{-0.3cm}
\end{table}

\section{Numerical Experiments}
\label{sec:experiments}
We evaluate regularized wasserstein kernels (RWKs) on graph classification benchmark tasks against the state-of-the-art baselines in order to answer the following questions: \vspace{0.05cm}
\begin{description}
    \item[\textbf{Q1.}] How well can RWK empirically perform for graph classification tasks?
    \item[\textbf{Q2.}] What impact do feature local variations have on the performance of RWK?
    \item[\textbf{Q3.}] How does each of the key components in RWK (i.e., different distance metrics and regularization terms) contribute to the overall performance of RWK?
    \item[\textbf{Q4.}] How efficiently can RWK perform in comparison with the existing optimal-transport based graph kernels?
\end{description} \vspace{0.05cm}
Below, we will present our experimental environment. Then, we will discuss the experimental results and answer these questions in Section \ref{sec:results}.

\subsection{Datasets}
In our experiments, we consider 12 benchmark datasets, which generally fall into two categories: 

\smallskip
\noindent\textbf{(1) Graphs with discrete attributes:} MUTAG, PTC-MR, NCI1, NCI109 and D\&D are bioinformatics datasets \cite{debnath1991structure, xu2018powerful, kriege2016valid, shervashidze2011weisfeiler}, and COLLAB is a social network \cite{yanardag2015deep} for which we use the same one-hot encoding setup as in \cite{xu2018powerful}. 

\smallskip
\noindent\textbf{(2) Graphs with continuous attributes:} COX2, COX2-MD, BZR, BZR-MD, PROTEINS and ENZYMES are bioinformatics datasets \cite{sutherland2003spline,borgwardt2005shortest,togninalli2019wasserstein}. 

Table \ref{Tab:datasets} provides further details about these datasets, including the availability of node and edge attributes, the number of graphs, and the number of classes. 
\begin{table}[h!]
\centering 
\scalebox{1}{\begin{tabular}{l c  c c c} 
\specialrule{.1em}{.05em}{.05em} 
\multirow{2}{*}{Dataset} &  Node & Edge & \multirow{2}{*}{\#Classes} &  \multirow{2}{*}{\#Graphs} \\  
& Attributes& Attributes& & \\[0.5ex] 
\hline 
MUTAG  & \checkmark &  - & 2 & 188\\
PTC-MR  & \checkmark &  - & 2 & 344 \\
NCI1  & \checkmark &  - & 2 & 4110 \\
D \& D  & \checkmark &  - & 2 & 1178\\
NCI109  & \checkmark &  - & 2 & 4127\\
COLLAB  & \checkmark &  - & 3 & 5000\\ \hline
ENZYMES  & \checkmark &  \checkmark & 6 & 600\\
PROTEINS  & \checkmark &  \checkmark & 2 & 1113 \\
COX2  & \checkmark &  \checkmark & 2 & 467\\
BZR  & \checkmark &  \checkmark & 2 & 405\\
COX2-MD  & \checkmark &  - & 2 & 303\\
BZR-MD  & \checkmark &  - & 2 & 306\\
\specialrule{.1em}{.05em}{.05em}
\end{tabular}}
\caption{Dataset statistics. \label{Tab:datasets}}
\end{table}

\begin{table*}
\centering 
\scalebox{1}{\begin{tabular}{r l c c c c c c}
& Method & MUTAG & PTC-MR & NCI1 & D\&D & NCI109 & COLLAB \\ [0.5ex] 
\hline
\multirow{5}{*}{\parbox{3cm}{Non-OT graph kernels}}
& WL & 90.4 $\pm$ 5.7 & 59.9 $\pm$ 4.3 & 86.0 $\pm$ 1.8 & 79.4 $\pm$ 0.3 & 85.9 $\pm$ 1.5 & 78.9 $\pm$ 1.9 \\
& WL-OA & 84.5 $\pm$ 1.7 & 63.6 $\pm$ 1.5 & 86.1 $\pm$ 0.2 & 79.2 $\pm$ 0.4 & 86.3 $\pm$ 0.2 & 80.7 $\pm$ 0.1 \\
& RetGK & 90.3 $\pm$ 1.1 & 62.5 $\pm$ 1.6 & 84.5 $\pm$ 0.2 & - & - & 81.0 $\pm$ 0.3 \\
& GNTK & 90.0 $\pm$ 8.5 & 67.9 $\pm$ 6.9 & 84.2 $\pm$ 1.5 & 75.6 $\pm$ 3.9 & - & 83.6 $\pm$ 1.0 \\
& P-WL & 90.5 $\pm$ 1.3 & 64.0 $\pm$ 0.8 & 85.4 $\pm$ 0.1 & 78.6 $\pm$ 0.3 & 84.9 $\pm$ 0.3 & - \\[0.5ex]\hline
\multirow{3}{*}{\parbox{3cm}{OT-based graph kernels}}& WL-PM & 87.7 $\pm$ 0.8 & 61.4 $\pm$ 0.8 & 86.4 $\pm$ 0.2 & 78.6 $\pm$ 0.2 & 85.3 $\pm$ 0.2 & 81.5 $\pm$  0.5 \\
& WWL & 87.2 $\pm$ 1.5 & 66.3 $\pm$ 1.2 & 85.7 $\pm$ 0.2 & 79.6 $\pm$ 0.5 & - & - \\
& FGW & 88.4 $\pm$ 5.6 & 65.3 $\pm$ 7.9 & 86.4 $\pm$ 1.6 & - & - & - \\[0.5ex]\hline
\multirow{4}{*}{\parbox{3cm}{GNN-based methods}} & PATCHY-SAN & 92.6 $\pm$ 4.2 & 60.0 $\pm$ 4.8 & 78.6 $\pm$ 1.9 & 77.1 $\pm$ 2.4 & - & 72.6 $\pm$ 2.2 \\ 
& DGCNN & 85.8 $\pm$ 0.0 & 58.6 $\pm$ 0.0 & 74.4 $\pm$ 0.0 & 76.6 $\pm$ 0.0 & 75.0 $\pm$ 0.0 & 73.7 $\pm$ 0.0 \\
& CapsGNN & 86.6 $\pm$ 1.5 & 66.0 $\pm$ 1.8 & 78.3 $\pm$ 1.3 & 75.3 $\pm$ 2.3 & 81.1 $\pm$ 3.1 & 79.6 $\pm$ 2.9 \\
& GIN & 89.4 $\pm$ 5.6 & 64.6 $\pm$ 7.0 & 82.7 $\pm$ 1.7 & 75.3 $\pm$ 3.5 & 86.5 $\pm$ 1.5 & 80.2 $\pm$ 1.9 \\[0.5ex]\hline
\multirow{3}{*}{\parbox{2cm}{Our work}}& RWK & \textbf{93.6} $\pm$ \textbf{3.7} & \textbf{69.5} $\pm$ \textbf{6.1} & \textbf{88.0} $\pm$ \textbf{4.5} & \textbf{81.6} $\pm$ \textbf{3.5} & \textbf{87.3} $\pm$ \textbf{6.1} & \textbf{83.8} $\pm$ \textbf{4.6}\\
& RWK-1 & 92.5 $\pm$ 3.1 & 68.9 $\pm$ 5.1 & 87.7 $\pm$ 6.1 & 81.0 $\pm$ 4.3 & 86.9 $\pm$ 5.2 & 83.2 $\pm$ 3.1\\
& RWK-0 & 90.7 $\pm$ 4.2 & 67.8 $\pm$ 3.6 & 87.0 $\pm$ 5.1 & 79.6 $\pm$ 3.1 & 86.4 $\pm$ 4.6 & 81.5 $\pm$ 3.9\\
\specialrule{.1em}{.05em}{.05em}
\end{tabular}}\caption{Classification accuracy (\%) averaged over 10 runs on graphs with discrete attributes. The results of WL and RetGK are taken from \cite{du2019graph} and the results of the other baselines are from their original papers.} 
\label{Tab:discrete-attributes-baselines} 
\end{table*}\vspace*{-0cm}
\begin{table*}[ht]
\centering 
\scalebox{1}{\begin{tabular}{r l c c c c c c c}
& Method & COX2 & ENZYMES & PROTEINS & BZR & COX2-MD & BZR-MD \\ [0.5ex] 
\hline
\multirow{4}{*}{\parbox{3cm}{Non-OT graph kernels}} & GHK & 76.4 $\pm$ 1.3 & 65.6 $\pm$ 0.8 & 74.7 $\pm$ 0.2 & 76.4 $\pm$ 0.9 & 66.2 $\pm$ 1.0 & 69.1 $\pm$ 2.0 \\
& PK & 77.6 $\pm$ 0.6 & 71.6 $\pm$ 0.5 & 61.3 $\pm$ 0.8 & 79.5 $\pm$ 0.4 & - & - \\
& HGK-WL & 78.1 $\pm$ 0.4 & 63.0 $\pm$ 0.6 & 75.9 $\pm$ 0.1 & 78.5 $\pm$ 0.6 & 74.6 $\pm$ 1.7 & 68.9 $\pm$ 0.6 \\
& HGK-SP& 72.5 $\pm$ 1.1 & 66.3 $\pm$ 0.3 & 75.7 $\pm$ 0.1 & 76.4 $\pm$ 0.7 & 68.5 $\pm$ 1.0 & 66.1 $\pm$ 1.0\\[0.5ex]\hline
\multirow{2}{*}{\parbox{3cm}{OT-based graph kernels}}& WWL & 78.2 $\pm$ 0.4 & 73.2 $\pm$ 0.8 & 77.9 $\pm$ 0.8 & 84.4 $\pm$ 2.0 & 76.3 $\pm$ 1.0 & 69.7 $\pm$ 0.9 \\
& FGW & 77.2 $\pm$ 4.8 & 71.0 $\pm$ 6.7 & 74.5 $\pm$ 2.7 & 85.1 $\pm$ 4.1 & - & - \\\hline
\multirow{3}{*}{\parbox{2cm}{Our work}}&RWK & \textbf{81.2} $\pm$ \textbf{5.3} & \textbf{78.3} $\pm$ \textbf{4.1} & \textbf{79.3} $\pm$ \textbf{6.1} & \textbf{86.2} $\pm$ \textbf{5.6} & \textbf{78.1} $\pm$ \textbf{4.3} & \textbf{71.9} $\pm$ \textbf{4.6}\\
&RWK-1 & 80.7 $\pm$ 4.6 & 77.5 $\pm$ 5.3 & 78.9 $\pm$ 4.5 & 85.8 $\pm$ 5.5 & 77.4 $\pm$ 3.7 & 71.3 $\pm$ 4.3\\
& RWK-0 & 79.6 $\pm$ 3.1 & 76.4 $\pm$ 4.5 & 78.2 $\pm$ 5.6 & 85.2 $\pm$ 4.3 & 76.7 $\pm$ 5.5 & 70.5 $\pm$ 3.7\\
\specialrule{.1em}{.05em}{.05em}
\end{tabular}}\caption{Classification accuracy (\%) averaged over 10 runs on graphs with continuous attributes. The results of GHK, HGK-WL and HGK-SP are taken from \cite{togninalli2019wasserstein} and the results of the other baselines are from their original papers.} 
\label{Tab:vector-attributes-baselines} \vspace*{-0.2cm}
\end{table*}

\subsection{Baseline Methods}We evaluate the performance of RWK against the following 16 state-of-the-art baselines, divided into three groups: \vspace{0.05cm} 
\begin{itemize}
\item[--] \textbf{Non-OT graph kernels:}
WL subtree kernel (WL) \cite{shervashidze2011weisfeiler}, WL Optimal Assignment Kernel (WL-OA) \cite{kriege2016valid}, Graph Hopper Kernel (GHK) \cite{feragen2013scalable}, Propagation Kernel (PK) \cite{neumann2016propagation}, Hash Graph Kernel (HGK-WL; HGK-SP) \cite{morris2016faster}, Return Probabilities of Random Walks Kernel (RetGK) \cite{zhang2018retgk}, Graph Neural Tangent Kernel (GNTK) \cite{du2019graph}, and Persistent WL Kernel (P-WL) \cite{rieck2019persistent}. 
\item[--] \textbf{OT-based graph kernels:} WL Pyramid Match Kernel (WL-PM) \cite{nikolentzos2017matching}, Wasserstein WL Graph Kernel (WWL) \cite{togninalli2019wasserstein} and Fused-Gromov Wasserstein (FGW) \cite{titouan2019optimal}. 
\item[--] \textbf{Graph Neural Network methods:} 
PATCHY-SAN  \cite{niepert2016learning}, 
Deep Graph Convolutional Neural Network (DGCNN) \cite{zhang2018end}, 
Capsule Neural Network (CapsGNN) \cite{xinyi2018capsule}, and Graph Isomorphism Network (GIN) \cite{xu2018powerful}. 
\end{itemize}

\subsection{Experimental Setup}
To benchmark the baseline methods, we follow the work of Titouan et al. \cite{titouan2019optimal} and use the same setup and data splits. The hyperparameters of our method are selected using the nested cross validation \cite{titouan2019optimal}. $C$-SVM classifier is used with $C \in \{10^{-5}, 10^{-4},$ $\dots, 10^{5}\}$. We choose the following parameter ranges: $\eta \in \{2^{-5},2^{-4},\dots, $ $2^{5}\}$, $\beta_1,\beta_2\in\{0.1, 0.2,\dots ,1\}$, $\lambda_{\mu},\lambda_{\nu},\lambda_g, \rho\in\{10^{-1}, 10^{-2}, \dots, 10^{-5}\}$, $t \in \{5, 10\}$, $b \in \{10, 20, \dots, 50\}$,  $\lambda \in \{0.1, 0.2, \dots, 0.9\}$, $\epsilon \in \{10^{-3},$ $ 10^{-4}, \dots, 10^{-9}\}$, and set $\alpha^{(0)}=0.99$ as the initial value of step size. For graphs with discrete attributes, we define feature similarity matrices on the
Weisfeiler-Lehman sequence of graphs \cite{weisfeiler1968reduction}. For BZR-MD and COX2-MD, we follow the same approach in \cite{togninalli2019wasserstein} to obtain node attributes. We consider the number of Weisfeiler-Lehman iterations $h\in\{1, 2\}$.

We choose the $l_2$ distance for $d_f$ and hamming distance for $d_s$. For the dimension of node embeddings, we set $k=64$. For feature local variation, we set $j=2$ (2-hop) as the default setting for RWK. The number and length of random walks are selected from $\{10,20,30\}$ and $\{2, 3, 4, 5, 6, 7, 8\}$, respectively. We train the model of node embeddings using 200 epochs and select the best learning rate from $\{10^{-4}, 10^{-3}, 10^{-2}, 10^{-1}\}$.

\vspace{0.2cm}
\section{Results and Discussion}\label{sec:results}
In this section, we discuss the experimental results to answer the aforementioned four questions.

\subsection{Graph Classification}
We first benchmark the performance of our RWK method against the baselines. The results are reported in Tables \ref{Tab:discrete-attributes-baselines}-\ref{Tab:vector-attributes-baselines}.

We see that, in Table \ref{Tab:discrete-attributes-baselines}, compared with the non-OT graph kernels, RWK improves upon their best results by a margin ranging from 0.2\% to 3.1\% on all datasets. Similarly, RWK improves upon the best results of the OT-based graph kernels by a margin ranging from 1.6\% to 5.2\% on all datasets, and upon the best results of the GNN-based baselines by a margin ranging from 0.8\% to 5.3\%.

In Table \ref{Tab:vector-attributes-baselines}, RWK also consistently performs better than all the baselines on all graphs with continuous attributes. Specifically, RWK improves upon the best results of the non-OT graph kernels by a margin ranging from 2.8\% to 6.7\% and the best results of the OT-based graph kernels by a margin ranging from 1.1\% to 5.1\% across the datasets.

It is worthy to mention that none of the baselines have achieved the best performance on all datasets, in comparison with the other baselines. However, in contrast, RWK consistently performs best on all datasets. Specifically, RWK improves upon the best results of the baselines by a margin of 1.0\% (PATCHY-SAN), 1.6\% (GNTK), 1.6\% (FGW), 2.0\% (WWL), 0.8\% (GIN), and 0.2\% (GNTK) on the datasets MUTAG, PTC-MR, NCI1, D\&D, NCI109 and COLLAB, respectively. A similar situation exists for graphs with continuous attributes.

\subsection{Impact of Local Variations}
To analyze the impact of feature local variations, we compare the performance of RWK that uses 2-hop feature local variations against the following two additional settings: \vspace{0.05cm}
\begin{itemize}
    \item[--] \textbf{RWK-0:} without using any feature local variations; 
    \item[--] \textbf{RWK-1:} with using 1-hop feature local variations.
\end{itemize}\vspace{0.05cm}

The results for this experiment are presented in Tables \ref{Tab:discrete-attributes-baselines}-\ref{Tab:vector-attributes-baselines}. We can see the following. First, feature local variations help further improve the performance considerably and consistently on all datasets, including both graphs with discrete attributes and graphs with continuous attributes. Second, on all these datasets, RWK consistently performs better than RWK-1, and RWK-1 consistently performs better than RWK-0. 

Nonetheless, in our experiments, we also notice that increasing the number of hops does not necessarily lead to improved performance due to the issue of oversmoothing. We thus restrict feature local variations within 2 hops.

\begin{table*}[ht]
\centering 
\scalebox{1}{\begin{tabular}{l c c c c c c c}
\specialrule{.1em}{.05em}{.05em} 
Variants \hspace*{1.5cm}&\hspace*{0.3cm} MUTAG\hspace*{0.3cm} & \hspace*{0.3cm}PTC-MR\hspace*{0.3cm} & \hspace*{0.3cm}NCI1\hspace*{0.3cm} & \hspace*{0.3cm}D\&D\hspace*{0.3cm} & \hspace*{0.3cm}NCI109\hspace*{0.3cm} & \hspace*{0.3cm}COLLAB\hspace*{0.3cm} \\ [0.5ex]
\hline
{NoLaplacianReg}  & 90.1 $\pm$ 3.5 & 67.0 $\pm$ 3.7 & 86.2 $\pm$ 5.3 & 79.4 $\pm$ 4.5 & 85.8 $\pm$ 5.2 & 81.5 $\pm$ 3.9 \\
{NoEntropyReg} & {92.2} $\pm$ {3.5} & {68.3} $\pm$ {6.5} & {87.3} $\pm$ {6.1} & {80.4} $\pm$ {3.6} & {86.5} $\pm$ {4.7} & {82.4} $\pm$ {3.8}\\
{NoRegs} & 88.9 $\pm$ 3.5 & 66.2 $\pm$ 4.6 & 85.3 $\pm$ 5.8 & 78.2 $\pm$ 3.9 & 84.7 $\pm$ 5.1 & 80.8 $\pm$ 4.1 \\ \hline
RWK-LW & 87.4 $\pm$ 4.2 & 64.8 $\pm$ 6.5 & 84.9 $\pm$ 3.6 & 77.8 $\pm$ 3.8 & 83.8 $\pm$ 5.7 & 79.5 $\pm$ 3.6 \\ 
RWK-GW & 82.8 $\pm$ 5.4 & 61.2 $\pm$ 5.8 & 81.9 $\pm$ 4.3 & 75.3 $\pm$ 4.8 & 80.7 $\pm$ 5.5 & 75.1 $\pm$ 3.9\\
\specialrule{.1em}{.05em}{.05em}
\end{tabular}}\caption{Classification accuracy (\%) averaged over 10 runs on graphs with discrete attributes.} 
\label{Tab:discrete-attributes-ablation} 
\centering 
\scalebox{1}{\begin{tabular}{l c c c c c c}
\specialrule{.1em}{.05em}{.05em} 
Variants\hspace*{1.5cm} &\hspace*{0.3cm}COX2\hspace*{0.6cm} & \hspace*{0.6cm}BZR\hspace*{0.4cm} & \hspace*{0.1cm}ENZYMES\hspace*{0.1cm} & \hspace*{0cm}PROTEINS\hspace*{0cm} & \hspace*{0.1cm}COX2-MD\hspace*{0.1cm} & \hspace*{0.1cm}BZR-MD\hspace*{0.1cm} \\ [0.5ex] 
\hline
NoLaplacianReg & 79.1 $\pm$ 3.9 & 84.8 $\pm$ 4.2 & 76.2 $\pm$ 3.8 & 77.5 $\pm$ 5.5 & 76.1 $\pm$ 4.6 & 68.7 $\pm$ 3.9\\
{NoEntropyReg} & {80.5} $\pm$ {5.4} & {85.7} $\pm$ {6.3} & {77.2} $\pm$ {3.7} & {78.5} $\pm$ {5.1} & 77.2 $\pm$ 4.1 & 69.8 $\pm$ 4.9\\
{NoRegs} & 78.2 $\pm$ 4.6 & 83.7 $\pm$ 5.6 & 75.4 $\pm$ 3.6 & 76.6 $\pm$ 4.8 & 75.9 $\pm$ 3.6 & 67.9 $\pm$ 4.5 \\ \hline
RWK-LW & 77.1 $\pm$ 4.1 & 82.8 $\pm$ 3.8 & 74.5 $\pm$ 5.2 & 75.5 $\pm$ 4.4 & 74.7 $\pm$ 4.3 & 66.8 $\pm$ 5.1\\
RWK-GW & 75.3 $\pm$ 5.4 & 79.6 $\pm$ 6.0 & 72.6 $\pm$ 3.3 & 73.2 $\pm$ 5.6 & 71.3 $\pm$ 4.1 & 64.1 $\pm$ 3.6\\
\specialrule{.1em}{.05em}{.05em}
\end{tabular}}\caption{Classification accuracy (\%) averaged over 10 runs on graphs with continuous attributes.} 
\label{Tab:continuous-attributes-ablation} \vspace*{-0.3cm}
\end{table*}
\begin{figure*}[ht!] 
 \centering
    \includegraphics[width=1\textwidth]{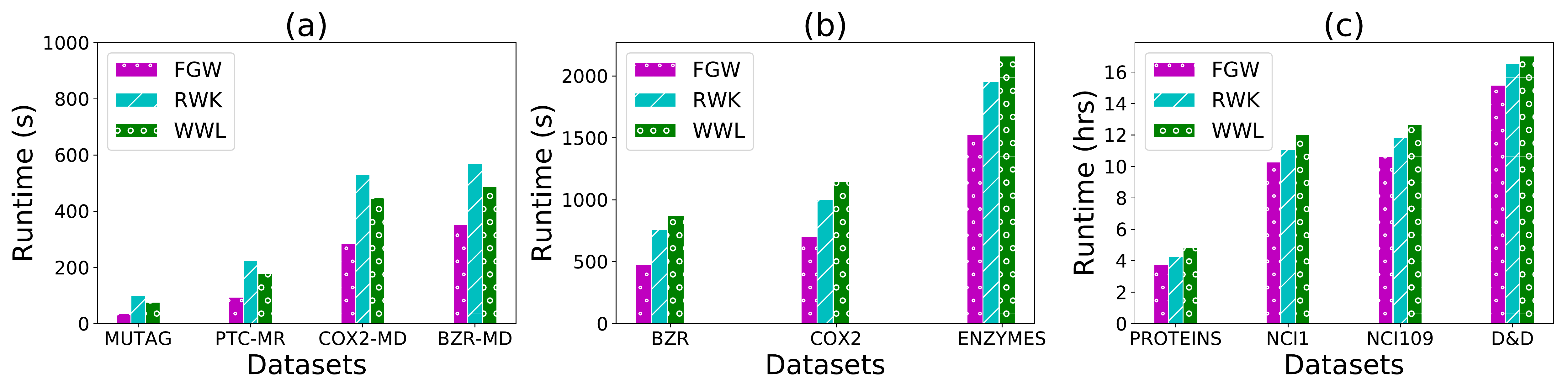}
    \caption{Running time averaged over 10 runs on graphs with discrete and continuous attributes. There are no result for the COLLAB dataset because all methods take more than 24 hours to obtain the results 
    }\label{fig:datasets_runtime_analysis}\vspace{-0.2cm}
\end{figure*}

\subsection{Ablation Analysis}
To demonstrate the effectiveness
of each component in the proposed method RWK, we conduct an ablation study on the following variants: \vspace{0.1cm}
\begin{itemize}
    \item \textbf{NoLaplacianReg:} This variant removes only the Laplacian regularization term $\Theta_w(\gamma)$ from RWK; 
    \item \textbf{NoEntropyReg:} This variant removes only the degree-entropy regularization term $\Theta_g(\gamma)$ from RWK; 
    \item \textbf{NoRegs:} This variant removes both regularization terms $\Theta_w(\gamma)$ and $\Theta_g(\gamma)$ from RWK; 
    \item \textbf{RWK-LW:} This variant removes only the global connectivity Wasserstein distance $GW(\mu,\nu)$ from RWK; 
    \item \textbf{RWK-GW:} This variant removes only the local barycentric Wasserstein distance $LW(\mu,\nu)$ from RWK. 
\end{itemize}\vspace{0.1cm}

The results are presented in Tables \ref{Tab:discrete-attributes-ablation}-\ref{Tab:continuous-attributes-ablation}. We observe that both local barycentric Wasserstein distance and global connectivity Wasserstein distance are crucial to the performance. The regularization terms $\Theta_g(\gamma)$ and $\Theta_w(\gamma)$ help reduce the performance variance while boosting the performance. Specifically, on graphs with discrete attributes, compared with RWK, the performance decreases by a margin ranging from 1.5\% to 3.5\% in NoLaplacianReg,  from 0.7\% to 1.4\% in NoEntropyReg, and from 2.6\% to 4.7\% in NoRegs. A similar trend exists on graphs with continuous attributes, where the performance decreases by a margin ranging from 1.4\% to 3.2\% in NoLaplacianReg, from 0.5\% to 2.1\% in NoEntropyReg, and from 2.2\% to 4.0\% in NoRegs. 
For RWK-LW, compared with RWK, the performance decreases by a margin ranging from 3.1\% to 6.2\% on graphs with discrete attributes and from 3.4\% to 5.1\% on graphs with continuous attributes. Similarly, for RWK-GW, the performance decreases by a margin ranging from 6.1\% to 10.8\% on graphs with discrete attributes and from 5.7\% to 7.8\% on graphs with continuous attributes.



\subsection{Runtime Analysis}
We evaluate the running time of RWK against the other OT-based graph kernel methods, i.e., FGW \cite{titouan2019optimal} and WWL \cite{togninalli2019wasserstein}. All these methods and our method were implemented in python. For a fair comparison, we do not consider WL-PM \cite{nikolentzos2017matching} because its implementation was done using MATLAB. Our experiments are performed on a Linux server which has 12-core Intel(R) Core(TM) i7-7800X CPU @ 3.50GHz, NVIDIA GeForce GTX Titan Xp with 96GB of main memory. The runtime results are averaged over 10 runs. 

Figure \ref{fig:datasets_runtime_analysis} shows the results. We see that: (1) FGW is the fastest one, compared with WWL and RWK, over all benchmark datasets; (2) RWK is slower than WWL on the 4 small datasets but faster than WWL on the other 7 larger datasets. This demonstrates the good scalability of RWK for large datasets. 
The reason why RWK is more scalable than WWL is as follows. WWL considers an unregularized Wasserstein optimization problem, which is usually cast as a linear programming problem and costly to solve \cite{cuturi2013sinkhorn}. In its algorithm implementation, WWL uses the EMD solver \cite{cuturi2013sinkhorn}. Different from WWL, RWK considers a regularized optimal transport problem, which is solved by our SCG algorithm being designed upon the \emph{Sinkhorn-knopp} matrix scaling for speeding up the computation.

\vspace*{-1cm}
\section{Conclusions}\label{sec:conclusions}
In this work, we have proposed a new optimal transport distance metric (i.e. RW discrepancy) on graphs in a learning framework for graph kernels. This optimal transport distance metric can preserve both local and global structures between graphs during the transport, in addition to preserving features and their local variations. Two strongly convex regularization terms were designed to theoretically guarantee the convergence and numerical stability in finding an optimal assignment between graphs. To empirically validate our method, we have evaluated our method against the state-of-the-art methods for graph classification, and have also analyzed the impact of feature local variations on the performance and the impact of each key component (including regularization terms) on the performance. The results have shown that our method outperforms all state-of-the-art approaches significantly in all benchmark tasks. In future, we plan to extend the current work to optimal transport based graph generative models to preserve global and local structure of generated graphs. 




\vspace{0.12cm}
\noindent \emph{Acknowledgement:} We gratefully acknowledge that the Titan Xp used for this research was donated by NVIDIA.
\bibliography{references}

\end{document}